\documentclass[oneside,reqno,11pt]{amsart}

\usepackage{geometry}
\usepackage[utf8]{inputenc} 
\usepackage[T1]{fontenc}    
\usepackage{hyperref}       
\usepackage{url}            
\usepackage{amsfonts}
\usepackage{amsthm}
\usepackage{amsmath}
\usepackage{amssymb}
\usepackage{todonotes}
\usepackage{bm}
\usepackage[toc,page]{appendix}
\usepackage{graphicx}
\usepackage{comment}
\usepackage{enumitem}
\setlist{nosep}

\newtheorem{theorem}{Theorem}[section]
\newtheorem*{theoremst}{Theorem}
\newtheorem{corollary}[theorem]{Corollary}
\newtheorem{lemma}[theorem]{Lemma}
\newtheorem{proposition}[theorem]{Proposition}

\theoremstyle{definition}
\newtheorem{definition}{Definition}[section]

\newtheorem{remark}{Remark}[section]

\newcommand{\sca}[2]{\langle #1 | #2\rangle}
\newcommand{\nr}[1]{\left\Vert #1\right\Vert}
\newcommand{\abs}[1]{\left\vert #1\right\vert}
\newcommand{\Nsp}{\mathbb{N}}
\newcommand{\Rsp}{\mathbb{R}}
\newcommand{\id}{\mathrm{id}}
\newcommand{\DD}{\mathrm{D}}
\newcommand{\Class}{\mathcal{C}}
\newcommand{\Wass}{\mathrm{W}}
\newcommand{\Meas}{\mathcal{M}}
\newcommand{\X}{\mathcal{X}}
\newcommand{\Y}{\mathcal{Y}}
\newcommand{\restr}[1]{\left.#1\right|}
\DeclareMathOperator{\diam}{diam}

\newcommand{\dd}{\mathrm{d}}

\renewcommand{\L}{\mathrm{L}}

\newcommand{\Lip}{\mathrm{Lip}}
\newcommand{\eps}{\varepsilon}
\newcommand{\TV}{\mathrm{TV}}
\newcommand{\Prob}{\mathcal{P}}
\newcommand{\vol}{\mathrm{vol}}

\begin{document}

\title[Quantitative stability of optimal transport maps]
      {Quantitative stability of optimal transport maps \\
        and linearization of the $2$-Wasserstein space}
\author{Quentin Mérigot}
\address{Laboratoire de Mathématiques d’Orsay, Univ. Paris-Sud, CNRS, Université Paris-Saclay, 91405 Orsay, France.}
\author{Alex Delalande}
\address{Inria Saclay - Ile-de-France}
\author{Frédéric Chazal}
\address{Inria Saclay - Ile-de-France}

\maketitle

\begin{abstract}
   This work studies an explicit embedding of the set of probability measures into a Hilbert space, defined using optimal transport maps from a reference probability density. This embedding linearizes to some extent the $2$-Wasserstein space, and enables the direct use of generic supervised and unsupervised learning algorithms on measure data.  Our main result is that the embedding is (bi-)Hölder continuous, when the reference density is uniform over a convex set, and can be equivalently phrased as a dimension-independent Hölder-stability results for optimal transport maps. 
\end{abstract}

\section{Introduction}

Numerous problems involve the comparison of point clouds, i.e. sets of points that lie in a metric space and for which the spatial distribution is of interest.  
Seeing the point clouds as discrete probability measures in a metric space, it is natural to compare them using Wasserstein distances defined by the optimal transport theory \cite{villani2003}.
These distances have indeed been successfully used in a variety of applications in machine learning \cite{ML_3, ML_2, ML_4, ML_6, ML_5, ML_1} and in statistics \cite{stat4, stat2, stat1, stat3}. In the discrete setting, many efficient algorithms have been proposed to compute or approximate the Wasserstein distances, such as Sinkhorn-Knopp and auction algorithms -- see \cite{comp_OT} and references therein. However efficient these algorithms are, 
they still represent a high computational costs when dealing with large databases of point clouds . For instance, when there are $k$ point clouds, $\frac{1}{2}k(k-1)$ optimal transport problems must be solved to get the distance matrix of the database. In addition, such methods provide good approximations of Wasserstein distances but 
they do not allow for the direct use of machine learning algorithms based on the Wasserstein geometry. In this work, we leverage the semi-discrete formulation of optimal transport to build an explicit embedding of finitely supported probability measures over $\Rsp^d$  into a Hilbert space.
This linear embedding allows one to directly apply supervised and unsupervised learning methods on point clouds datasets consistently with the Wasserstein geometry, thus alleviating the non-Hilbertian nature of Wasserstein spaces in dimensions greater than $2$ (see for instance Section 8.3 in \cite{comp_OT}).



\subsection{Optimal transport and Monge maps}
Let $\X,$ $\Y$ be two compact and convex subsets of
$\Rsp^d$. Let $\rho$ be a probability density on $\X$
and $\mu$ be a probability measure on $\Y$. We consider the
squared Euclidean cost $c(x,y) := \|x-y\|^2$ for all $x, y \in
\Rsp^d$. Monge's formulation of the optimal transport problem
consists in minimizing the transport cost over all transport maps
between $\rho$ and $\mu$, that is
\begin{equation}
    \label{eq:ot}
    \text{\small $\min_{T} \left\{ \int_{\X} c(x, T(x)) \rho(x) \dd(x)  \mid  T : \X \rightarrow \Y,  T_{\#}\rho = \mu \right\}$,}
\end{equation}
where $T_\#\rho$ is the \emph{pushforward measure}, defined by
\begin{equation}\forall B\subseteq \Y,~~T_\#\rho(B) = \rho(T^{-1}(B)).
\end{equation}
By the work of Brenier \cite{Brenier}, this problem
admits a solution that is uniquely defined as the gradient $T=\nabla \phi$ of a convex
function $\phi$ on $\X$ referred to in what follows as a
\textit{Brenier potential}. Here, we will refer to the map $T$ as the \emph{Monge map}.
In this work, the source probability density $\rho\in\Prob(\X)$ is fixed once and for all.

\begin{definition}[Monge embedding]
Given any probability  measure $\mu$ on $\Y$, we denote $T_\mu$ the solution of the optimal transport problem   (\ref{eq:ot}) between $\rho$ and $\mu$. We call \textit{Monge embedding} the mapping 
\begin{equation} \label{eq:monge-embedding}
    \begin{aligned}
    \Prob(\Y) &\to L^2(\rho,\Rsp^d) \\
    \mu &\mapsto T_\mu,
    \end{aligned}
\end{equation}
where $\Prob(\Y)$ is the set of probability measures over $\Y$.
\end{definition}

An attractive feature of the Monge embedding is that the map $T_\mu$ can be efficiently computed when $\mu$ is finitely supported  on $\Rsp^2$ or $\Rsp^3$ \cite{kitagawa:hal-01290496}, see also references in Remark \ref{rem:na}. In higher dimension, they can also be estimated using stochastic optimization methods \cite{genevay:hal-01321664}.

\subsection{Contributions} 
Our main interest in this work is the regularity properties of the Monge embedding \eqref{eq:monge-embedding}, or equivalently the stability of the optimal transport map in terms of the target measure. Our main theorem shows that the Monge map is a bi-Hölder embedding of $\Prob(\Y)$ endowed with the Wasserstein distance $\Wass_p$ (defined in equation \eqref{eq:wass}) into the Hilbert space $L^2(\rho,\Rsp^d)$. More importantly, we show that the Hölder exponent does not depend on the ambient dimension $d$. 
\begin{theoremst}[Theorem~\ref{th:map-holder-tv}] Let $\rho$ be the Lebesgue measure on a compact convex subset $\X$ of $\Rsp^d$ with unit volume, and let $\Y$ be a compact convex set. Then, 
for all $\mu,\nu\in\Prob(\Y),$ and all $p\geq 1$,
\begin{equation}
  \label{eq:main} \Wass_2(\mu,\nu) \leq \|T_\mu - T_\nu\|_{\L^2(\rho)} \leq C \Wass_p(\mu,\nu)^{\frac{2}{15}},
\end{equation}
where the constant $C$ depends on $d$, $\X$ and $\Y$.
\end{theoremst}
The upper bound of this theorem should be compared to Theorem~\ref{th:hold-cont-reg} (similar to a result of Ambrosio reported in \cite{gigli2011holder}), which shows a $\frac{1}{2}$-Hölder behaviour under a very strong regularity assumption on $T_\mu$, and  to Corollary \ref{coro:Berman} (from Berman, see \cite{Berman2018ConvergenceRF}), which holds without assumption on $\mu,\nu$, but whose exponent scales exponentially badly with the dimension $d$. We conclude the article by  illustrations of the behavior of this embedding, and we showcase a few applications.

\begin{remark}[Geometric interpretation]
Similarly to \cite[Eq.~(3)]{LOT_ref_image} or
\cite[\S10.2]{ambrosio2008gradient}, one can define a distance on $\Prob(\Y)$ using the
formula
\begin{equation}
  W_{2,\rho}(\mu, \nu) := \|T_\mu - T_\nu\|_{\L^2(\rho)},
\end{equation}
and our main results reads as a bi-Hölder equivalence between this distance and the $2$-Wasserstein distance:
\begin{equation}
    \label{eq:biHolder-dist_W2rho}
    \Wass_2(\mu,\nu) \leq W_{2,\rho}(\mu, \nu) \leq C \Wass_2(\mu,\nu)^{\frac{2}{15}}.
\end{equation}
\end{remark} 

\begin{remark}[Numerical analysis interpretation]
\label{rem:na}
  The second inequality in \eqref{eq:main} (and similar stability
  results for optimal transport maps) also has consequences in
  numerical analysis. Indeed, a natural way to approximate the
  solution of the optimal transport problem between a probability
  density $\rho$ and a measure $\mu$ consists in approximating $\mu$
  by a sequence of measures $(\mu_N)_{N\geq 1}$ with finite support
  such that $\lim_{N\to+\infty} \Wass_1(\mu,\mu_N) = 0$, and to
  approximate $T_\mu$ by $T_{\mu_N}$. This is the so-called
  \emph{semi-discrete approach}, which can be traced back to Minkowski
  and Alexandrov and developed in many works from the 1990s
  \cite{cullen1991generalised,gangbo1996geometry,oliker1989numerical,caffarelli1999problem}. This
  approach was revisited and popularized by recent development of
  efficient algorithms to solve semi-discrete optimal transport
  problems
  \cite{aurenhammer1998minkowski,merigot2011multiscale,de2012blue,levy2015numerical,kitagawa:hal-01290496,genevay:hal-01321664}.

  The convergence of $T_{\mu_N}$ to $T_\mu$ is well-known in optimal
  transport, but it usually follows from an abstract and general
  stability results for optimal transport maps (see
  Proposition~\ref{prop:monge-embedding}), which uses compactness
  arguments and is therefore non-quantitative. Our main theorem gives
  Hölder convergence rates for the convergence of $T_{\mu_N}$ to
  $T_\mu$, which are in addition dimension independent. 
\end{remark}

\subsection{Related work in statistics and learning} 
The Monge embedding \eqref{eq:monge-embedding} was introduced in
\cite{LOT_ref_image} in the context of pattern recognition in images,
were the problem of computing a distance matrix based on
transportation metrics over a possibly large dataset of images is
tackled. The approach proposed in \cite{LOT_ref_image} computes a
reference image as a mean image (for the $2$-Wasserstein distance) of
the whole dataset and then computes the OT maps between this reference
image $\rho$ and each image $\mu_i$ of the training set. Distances
between images are then defined based on Euclidean distances between
these maps.

The geometric idea comes from a Riemannian interpretation of the
Wasserstein geometry \cite{otto2001geometry,ambrosio2008gradient}. In
this interpretation, the tangent space to $\Prob(\Rsp^d)$ at $\rho$ is
included in $L^2(\rho,\Rsp^d)$. The optimal transport map $T_{\mu_i}$
between $\rho$ and $\mu_i$ can then be regarded as the vector in the
tangent space at $\rho$ which supports the Wasserstein geodesic from
$\rho$ to $\mu_i$. Thus Monge's embedding sends any probability
measure $\mu_i$ in the (curved) manifold $\Prob(\Rsp^d)$ to a vector
$T_{\mu_i}$ belonging to the \emph{linear space} $L^2(\rho,\Rsp^d)$,
which retain some of the geometry of the space. In the Riemannian
language again, the map $\mu\mapsto T_\mu$ would be called a
\emph{logarithm}, i.e. the inverse of the Riemannian exponential
map. This establishes a connection between this idea and similar
strategies used to extend statistical inference notions (such as PCA)
to manifold-valued data, e.g. \cite{fletcher2004principal,stat2}.

The work in \cite{chernozhukov2017} also proposes to use OT maps in a
statistical context to overcome the lack of a canonical ordering in
$\Rsp^d$ for $d > 1$. Notions of vector-quantile, vector-ranks and
depth are defined based on the transport maps (and there inverses)
between a reference measure defined as the uniform distribution on the
unit hyperball and the $d$-dimensional samples of interest.
Monge maps are also studied in \cite{htter2019minimax} where an
estimator for such maps between population distributions is proposed
when only samples from the distributions of interest are
available. Minimax estimation rates for (very) smooth transport maps
in general dimension are given and the proposed estimator is shown to
achieve near minimax optimality.


\section{Known properties of the Monge embedding}

\subsection{Assumptions and notations} 
From now on, we fix two compact convex subsets $\X,\Y$ of $\Rsp^d$,
and we fix once and for all a probability density $\rho$ on $\X$. For
simplicity, we assume that the support of $\rho$ equals $\X$. We also
denote $M_\X \geq 0$ the smallest positive real such that $\X \subset
B(0, M_\X)$, and $\diam(\X)$ the \emph{diameter} of the set $\X$, and similarly for $\Y$.

We will use the  notion of potentials associated to the optimal transport problem 
throughout the article.
\begin{definition}[Potentials] \label{def:potentials}
Given a measure $\mu$, we denote $T_\mu$ the Monge map, we denote
$\phi_\mu$ the convex \emph{Brenier potential} so that $T_\mu = \nabla
\phi_\mu$. We define the \emph{dual potential} $\psi_\mu$ on $\Y$ as
the Legendre transform of $\phi_\mu$:
\begin{equation}\label{eq:lega}
\psi_\nu(y) = \max_{x\in\X} \sca{x}{y} - \phi_\nu(x).
\end{equation}
Adding a constant to $\phi_\nu$ if necessary, we assume that $\int \psi_\nu \dd
\nu = 0$.
\end{definition}

\begin{remark}[Uniqueness and estimates]\label{rem:uniqueness}
The two potentials $(\phi_\nu,\psi_\nu)$ are closely related the \emph{Kantorovich
  potentials} associated to the optimal transport problem
\eqref{eq:ot}. In our setting, where the support of $\rho$ is the
whole domain $\X$, these potentials are unique up to
addition of a constant \cite[Proposition
  7.18]{santambrogio2015optimal}, and the constant is fixed using the
condition $\int \psi_\nu \dd\nu = 0$.

By Eq.\eqref{eq:lega}, we have the following Lipschitz estimate on $\psi_\mu$,
\begin{equation}\label{eq:lipschitz}
\Lip(\psi_\mu) \leq M_\X,
\end{equation}
where $\Lip(f)$ denotes the Lipschitz constant of $f$.
The assumption $\int \psi_\mu \dd\mu = 0$ implies that $\psi_\mu$ takes non-negative and non-positive values, implying that 
\begin{equation}\label{eq:linfty}
\nr{\psi_\mu}_\infty \leq \Lip(\psi_\mu) \diam(\Y) = M_\X \diam(\Y).
\end{equation}
\end{remark}

\subsection{Elementary properties}
A first obvious property of the embedding $\mu \mapsto T_{\mu}$ is its injectivity: if $\mu$ and $\nu$ are measures on $\Y$ such that $T_{\mu} = T_{\nu}$, then
$(T_{\mu})_{\#}\rho = \mu = (T_{\nu})_{\#}\rho = \nu$. This  injectivity  ensures that the Monge embedding preserves the discriminative information about the measures it embeds. A stronger formulation of this injectivity property can be made using  Wasserstein distance.
\begin{definition}[Wasserstein distance] The Wasserstein distance of exponent $p$
between $\mu,\nu\in\Prob(\Y)$ is defined by
\begin{equation} \label{eq:wass}
    \Wass_p^p(\mu, \nu) = \inf_{\gamma \in \Pi(\mu, \nu)} \int_{\Y\times \Y} \|y - y'\|^p d\gamma(y,y')
\end{equation}
where $\Pi(\mu, \nu) = \{ \gamma \in \Meas(\Y \times \Y) \mid \forall A \subset \Y, \gamma(A\times \Y) = \mu(A), \gamma(\Y\times A) = \nu(A) \}$.
\end{definition}
\begin{remark}
Jensen's inequality gives $\Wass_1\leq \Wass_p$, showing that $\Wass_1$ is the weakest Wasserstein
distance. On the other hand, since $\Y$ is bounded, $\Wass_p$ can also be bounded in terms of $\Wass_1$ (see \cite[Eq. (5.1)]{santambrogio2015optimal}):
\begin{equation}
  \forall \mu,\nu\in\Prob(\Y),~~ \Wass_1(\mu,\nu) \leq \Wass_p(\mu,\nu) \leq \diam(\Y)^{\frac{p-1}{p}} \Wass_1(\mu,\nu)^{\frac{1}{p}},
\end{equation}
showing that all Wasserstein distances are in fact bi-Hölder equivalent.
\end{remark}

\begin{proposition} \label{prop:monge-embedding} The following properties hold:
\begin{enumerate}
    \item[(i)] The Monge embedding is reverse-Lipschitz:
    \begin{equation}
\forall \mu,\nu\in\Prob(Y),~~\Wass_2(\mu, \nu) \leq \|T_{\mu} - T_{\nu}\|_{\L^2(\rho)}.    
\end{equation}
\item[(ii)]  The Monge embedding is continuous.
\item[(iii)] The Monge embedding is in general not better than $\frac{1}{2}$-Hölder wrt $\Wass_2$.
\end{enumerate}
\end{proposition}

We note again that the proof of the general continuity result (ii)
uses compactness arguments and is not quantitative. Our goal in the
next two sections is to study the Hölder-continuity of the Monge map
embedding with respect to the $1$-Wasserstein distance (which is the
weakest Wasserstein distance) and to the total variation distance
between measures.

\begin{proof}[Proof of Proposition~\ref{prop:monge-embedding}]
  If we denote $\gamma := (T_{\mu}, T_{\nu})_{\#} \rho$, then $\gamma \in \Pi(\mu, \nu)$. The change of variable formula gives
\begin{align*}
    \text{\small $\Wass_2^2(\mu, \nu)$} &\text{\small $\leq \int_{\Y\times \Y} \|y - y'\|_2^2 d\gamma(y,y')$} \\
    &= \int_{\X} \|T_{\mu}(x) - T_{\nu}(x)\|_2^2 \rho(x) \dd x  = \|T_{\mu} - T_{\nu}\|^2_{\L^2(\rho)}, 
\end{align*}
showing (i).  The continuity (ii) of the map $\mu\mapsto T_\mu$
follows from e.g. Exercise 2.17 in \cite{villani2003}. To prove (iii),
we use the following lemma.
\end{proof}

\begin{lemma} \label{lemma:onehalf}
Let $\rho$ be uniform on the unit disc $\X \subseteq \Rsp^2$. Then,
there is a curve $\theta\in [0,2\pi] \to\mu_\theta \in \Prob(\X)$ and $C>0$ such that 
$\nr{T_{\mu_\theta} - T_{\mu_0}}_{\L^2(\rho)} \geq C \Wass_2(\mu_\theta,\mu_0)^{1/2}$.
\end{lemma}

\begin{proof}
 Given $\theta \in \Rsp$, we denote $x_\theta = (\cos \theta,\sin(\theta))$ and $\mu_\theta = \frac{1}{2}(\delta_{x_\theta} + \delta_{-x_\theta})$. Then, the optimal transport map
between $\rho$ and $\mu_\theta$ is given by 
\begin{equation} T_{\mu_\theta}(x) =
\begin{cases}
x_\theta & \hbox{ if } \sca{x}{x_\theta}\geq 0 \\
-x_\theta & \hbox{ if not.}
\end{cases}
\end{equation}
One can easily check that for $\theta$ one has 
$\Wass_2(\mu_0, \mu_\theta) \leq \abs{\theta}$. For $\theta > 0$ we set 
\begin{equation} D_\theta = \{ x\in \Rsp^2 \mid \sca{x}{x_0} \geq 0 \hbox{ and } \sca{x}{x_\theta} \leq 0\}.
\end{equation}
Then, on $D_\theta$, $T_{\mu_\theta} \equiv x_{-\theta}$ and $T_{\mu_0} \equiv x_0$, giving
\begin{equation} \nr{T_{\mu_\theta} - T_{\mu_0}}^2_{\L^2(\rho)} \geq \int_{D_\theta} \nr{x_{-\theta} - x_0}^2 \dd x = \abs{D_\theta} \nr{x_{-\theta} - x_0}^2. 
\end{equation}
Moreover, if $\abs{\theta} \leq \frac{\pi}{2}$ one has $\nr{x_{-\theta} - x_0}^2 \geq 2$. This gives
\begin{equation*}
  \nr{T_{\mu_\theta} - T_{\mu_0}}^2_{\L^2(\rho)}\geq 2\abs{D_\theta} \geq \frac{\abs{\theta}}{\pi}. \qedhere
  \end{equation*}
\end{proof}

\subsection{Hölder-continuity near a regular measure}

We state a first result, which is a slight variant of a known stability result due to Ambrosio and reported in \cite{gigli2011holder}. While \cite{gigli2011holder} shows a local $1/2$-Hölder behaviour for regular enough source and target measures along a curve in the $2$-Wasserstein space, we show the same Hölder behaviour near a probability measure $\mu$ whose Monge map $T_\mu$ is Lipschitz continuous, but with respect to the $1$-Wasserstein distance. 

\begin{theorem}
\label{th:hold-cont-reg}
Let $\mu,\nu \in \Prob(\Y)$ and assume that $T_\mu$ is $K$-Lipschitz. Then,
\begin{equation}
  \|T_\mu - T_\nu\|_{\L^2(\rho)} \leq 2 \sqrt{M_\X K} \Wass_1(\mu, \nu)^{1/2}.
\end{equation}
\end{theorem}

We deduce this theorem from the following elementary lemma.

\begin{lemma}
\label{lemma:reg}
Under the assumptions of Theorem~\ref{th:hold-cont-reg},
\begin{equation}
  \|T_\mu - T_\nu\|_{\L^2(\rho)}^2 \leq 2 K \int_{\Y} (\psi_\nu - \psi_\mu)\dd(\mu - \nu)
\end{equation}
\end{lemma}
\begin{proof}
From convex analysis, we know that the map $T_\mu = \nabla \phi_\mu$
is K-Lipschitz if and only if $\psi_\mu$ defined in \eqref{eq:lega} is
$\frac{1}{K}$-strongly convex. We denote $A = \int_{\Y}
\psi_\nu \dd(\mu - \nu)$ and $B = \int_{\Y} \psi_\mu \dd(\nu
- \mu)$. Using that $(\nabla \phi_\mu)_{\#} \rho = \mu$ and  $(\nabla \phi_\nu)_{\#} \rho = \nu$, we get 
\begin{align}
  A &= \int_{\X} (\psi_\nu(\nabla \phi_\mu) - \psi_\nu(\nabla \phi_\nu)) \dd\rho \notag \\
  &= \int_{\X} (\psi_\nu(\nabla \psi_\mu^*) - \psi_\nu(\nabla \psi_\nu^*)) \dd\rho
\end{align}
We now use the inequality $\psi_\nu(y) - \psi_\nu(z)\geq
\sca{y-z}{v}$, which holds for all $v$ in the subdifferential
$\partial\psi_\nu(z)$. The convex functions $\psi_\nu,\psi_\mu$ are
differentiable $\rho$-almost everywhere. Taking $z = \nabla
\psi_\nu^*(x)$ and $y = \nabla \psi_\mu^*(x)$, and using $x \in
\partial \psi_\nu(z)$, we obtain
\begin{equation}
A \geq \int_{\X} \langle \id, \nabla \psi_\mu^* - \nabla \psi_\nu^* \rangle \dd\rho
\end{equation}
Using the strong convexity of $\psi_\mu$, we get a similar lower bound on $B$,
with an extra quadratic term
\begin{align}
    B =& \int_{\X} (\psi_\mu(\nabla \psi_\nu^*) - \psi_\mu(\nabla \psi_\mu^*)) \dd\rho
    \notag\\
    &  \geq \int_{\X} \big(\langle \id, \nabla \psi_\nu^* - \nabla \psi_\mu^* \rangle + \frac{1}{2K}\|\nabla \psi_\nu^* - \nabla \psi_\mu^*\|_2^2\big) \dd\rho.
\end{align}
Summing up the lower bounds on $A$ and $B$, we get:
\begin{align*}
    \int_{\Y} (\psi_\nu - \psi_\mu)d(\mu - \nu) \geq &\frac{1}{2K} \int_{\X} \|\nabla \psi_\nu^* - \nabla \psi_\mu^*\|_2^2 \dd\rho \\
    &= \frac{1}{2K} \|T_\nu - T_\mu\|_{\L^2(\rho)}^2. \quad \qedhere
\end{align*}
\end{proof}

\begin{proof}[Proof of Theorem~\ref{th:hold-cont-reg}]
Combining the Lipschitz estimate \eqref{eq:lipschitz} this with Lemma \ref{lemma:reg}, 
\begin{align} 
\|T_\mu - T_\nu\|_{\L^2(\rho)}^2 &\leq 2 K \int_{\Y} (\psi_\nu - \psi_\mu)\dd(\mu - \nu) 
\notag\\
&\leq 2K \max_{\Lip(f)\leq M_\X} \int_{\Y} f \dd(\mu - \nu) \notag\\
&= 2K M_\X \max_{\Lip(f)\leq 1} \int_\Y f \dd(\mu - \nu)  \notag\\
&= 2K M_{\X} \Wass_1(\mu,\nu),
\end{align}
where we used Kantorovich-Rubinstein's theorem to get the last equality.
\end{proof}

\subsection{Dimension-dependent Hölder continuity}
\label{sec:holder-general}

Here  we assume that $\rho \equiv 1$ on a compact convex set $\X$ with unit volume. With no assumption on the embedded measures $\mu$ and $\nu$, another Hölder-continuity result for Monge's embedding, can be derived from the following theorem of Berman \cite{Berman2018ConvergenceRF}.

\begin{theorem}[\cite{Berman2018ConvergenceRF} Proposition 3.4] \label{th:Berman-original}
For any measures $\mu$ and $\nu$ in $\Prob(\Y)$,  
\begin{equation} \| \nabla \psi_\mu - \nabla \psi_\nu \|^2_{\L^2(\Y)} \leq C \Big(\int_{\Y} (\psi_\nu - \psi_\mu)\dd(\mu - \nu)\Big)^{\frac{1}{2^{d-1}}}, 
\end{equation}
where $C$  depends only on $\rho$, $\X$ and $\Y$.
\end{theorem}

We deduce a global Hölder-continuity result for the Monge embedding
\eqref{eq:monge-embedding}. Note however that the Hölder exponent
depends on the ambient dimension $d$, and the dependence is
exponential. The proof of this corollary is in the appendix.

\begin{corollary}\label{coro:Berman}
For any measures $\mu$ and $\nu$ in $\Prob(\Y)$,
\begin{equation}
  \|T_\mu - T_\nu\|_{\L^2(\rho)} \leq C \Wass_1(\mu, \nu)^{\frac{1}{2^{(d-1)}(d+2)}}
\end{equation}
where $C$ depends only on $\rho$, $\X$ and $\Y$
\end{corollary}

\section{Dimension-independent Hölder-continuity of the Monge embedding}
\label{sec:stab-TV}

This section is devoted to a global stability result for the Monge map
embedding. As in \S\ref{sec:holder-general}, we require that the
source measure is the Lebesgue measure $\rho \equiv 1$ on some compact
convex domain $\X$ with unit volume, and that $\Y$ is bounded.
Unlike Theorem~\ref{th:hold-cont-reg}, this stability result does not
make any regularity assumption on the measures $\mu,\nu$. In addition,
the Hölder exponent does not depend on the ambient dimension, unlike
Corollary~\ref{coro:Berman} of the previous section. We also report a
stability of $\mu \mapsto T_\mu$ with respect to the total variation
(TV) distance. This distance is much stronger than the Wasserstein
distance, but the Hölder-exponent we obtain is slightly better.

\begin{theorem}[Stability of Monge maps]
\label{th:map-holder-tv} The following inequalities hold for all probability measures $\mu,\nu$ on a bounded set $\Y$
\begin{align}
\nr{T_\nu - T_\mu}_{\L^2(\X)} &\leq C \nr{\nu - \mu}_{\TV}^{1/5}.\\
 \nr{T_\nu - T_\mu}_{\L^2(\X)} &\leq C \Wass_1(\mu,\nu)^{2/15},
 \end{align}
where $C$ only depend on $d,\X$ and $\Y$.
\end{theorem}

The proof of this stability theorem is deduced from the stability of
dual potentials, which may be interesting in its own.

\begin{theorem}[Stability of dual potentials]
  \label{th:stab-dual}
 Let $\mu^0,\mu^1\in \Prob(Y)$ and let $\psi^0, \psi^1$ be the associated dual potentials (see Def.~\ref{def:potentials}). Then,
\begin{align}
&\nr{\psi^1 - \psi^0}_{\L^2(\mu^0+\mu^1)} \leq C \nr{\mu^1 - \mu^0}^{\frac{1}{2}}_{\TV},\\ 
&\nr{\psi^1 - \psi^0}_{\L^2(\mu^0+\mu^1)} \leq C \Wass_1(\mu^1, \mu^0)^{\frac{1}{3}}, \end{align}
where $C$ only depends on $d,\X$ and $\Y$
\end{theorem}

\begin{remark}[Non-optimality]
The Hölder-exponent $\frac{2}{15}$ in Theorem~\ref{th:map-holder-tv}
comes from the proof, but we see no reason why it should be the
optimal exponent. Combining Theorem~\ref{th:map-holder-tv} with
Proposition~\ref{prop:monge-embedding}.(iii), we see that the best
exponent belongs to the range $[\frac{2}{15},\frac{1}{2}]$.
\end{remark}

\begin{remark}[Brenier embedding]
Instead of working with the optimal transport maps $T_\mu$, one could
also directly work with the Brenier potentials $\phi_\mu \in
\L^2(\X)$. A straightforward modification of the proof of
Theorem~\ref{th:map-holder-tv} shows Hölder-continuity of the map $\mu
\in \Prob(\Y)\mapsto\phi_\mu \in L^2(\X)$, with slightly improved
exponents: the exponent would be $1/3$ with respect to the Wasserstein
distance and $2/9$ with respect to the TV distance.
\end{remark}

\begin{remark}[McDiarmid's inequality]
Assume that $\mu_N,\nu_N$ are uniform on point clouds with $N$ points, and that their support has $N-1$ common points, i.e.
\begin{equation}\mu_N = \frac{1}{N} \sum_{1\leq i\leq N}\delta_{y_i},\qquad  \nu_N = \frac{1}{N}\sum_{2\leq i\leq N+1} \delta_{y_i}.
\end{equation}
Then, the theorem gives $\nr{T_{\mu_N} - T_{\nu_N}}\leq C N^{-1/5}$. This shows
that if one considers the function
\begin{equation}
 f(y_1,\hdots,y_N) = \|T_{\mu} - T_{\frac{1}{N}\sum_i \delta_{y_i}}\|_{\L^2(\rho)},
\end{equation}
then,
\begin{equation}
\begin{aligned}|&f(y_1,\hdots,y_{i-1},y_i,y_{i+1}, \hdots,y_N) - \\&f(y_1,\hdots, y_{i-1}, \hat{y}_i, y_{i+1},\hdots, y_N)| \leq \frac{C}{N^{-1/5}}. 
\end{aligned}
\end{equation}
This bound is in $N^{-1/5}$, which ensures some statistical
consistency but is not enough to get concentration results with a direct use of McDiarmid's inequality -- one would need a
bound in $N^{-\alpha}$ with $\alpha > \frac{1}{2}$ to use this
inequality and deduce concentration results from it.
\end{remark}

\subsection{From the semi-discrete to the general case}
We will establish Theorem~\ref{th:stab-dual} in the case  where both  measures  $\mu^0$ and $\mu^1$ are supported on the same set, which is finite. We show in this section that  the general case can be deduced. This follows from a simple density argument, which is summarized in the following lemma.

\begin{lemma} \label{lemma:density} Given any $\mu^0,\mu^1\in\Prob(\Y)$, there
exists sequences $(\mu^k_N)_{N\geq 1}$ such that 
\begin{itemize}
\item $\mu^0_N$ and $\mu^1_N$ have the same support, which is finite,
\item $\lim\sup_{N\to +\infty} \|\mu_N^0 - \mu_N^1\|_\TV \leq \|\mu^0 - \mu^1\|_\TV $,
\item $\lim_{N\to +\infty} \Wass_1(\mu^0_N, \mu^1_N) = \Wass_1(\mu^0, \mu^1)$,
\item Denote $\psi^k = \psi_{\mu^k}$ and $\psi^k_N = \psi_{\mu^k_N}$. Then,
\begin{equation}\lim_{N\to +\infty} \nr{\psi^1_N - \psi^0_N}_{\L^2(\mu^0_N + \mu^1_N)} = 
\nr{\psi^1 - \psi^0}_{\L^2(\mu^0 + \mu^1)}.
\end{equation}
\end{itemize}
\end{lemma}

\begin{proof}
For any $N>0$, we consider a finite partition  $\Y = \sqcup_{1\leq i\leq N} \Y^N_i$, we let $\eps_N = \max_i \diam(\Y_i^N)$ and we assume that $\lim_{N\to +\infty}\eps_N = 0$. Then, we define
\begin{equation} \mu^k_N = \sum_{1\leq i\leq N} \left[\left(1-\frac{1}{N}\right) \mu^k(\Y_i^N) + \frac{1}{N^2}\right] \delta_{y_i^N}, 
\end{equation}
where $y_i^N \in \Y_i^N$. Then, it is easy to check that the support of the measures $\mu^0_N$ and $\mu^1_N$ is the set $\{y_1^N,\hdots, y^N_N\}$. Moreover, 
\begin{equation} \nr{\mu^1_N - \mu^0_N}_\TV \leq \nr{\mu^1 - \mu^0}_\TV.
\label{eq:tvstab}
\end{equation}
In addition,  $\Wass_1(\mu^k_N,\mu^k) \leq \eps_N \xrightarrow{N\to +\infty} 0.$
Combined with the triangle inequality, we deduce
\begin{align}
|\Wass_1(\mu^0_N,\mu^1_N) - \Wass_1(\mu^0,\mu^1)| &= |\Wass_1(\mu^0_N,\mu^1_N) - \Wass_1(\mu^0_N,\mu^1) + \Wass_1(\mu^0_N,\mu^1) - \Wass_1(\mu^0,\mu^1)| \notag\\
&\leq |\Wass_1(\mu^0_N,\mu^1_N) - \Wass_1(\mu^0_N,\mu^1)| + |\Wass_1(\mu^0_N,\mu^1) - \Wass_1(\mu^0,\mu^1)| \notag\\
&\leq \Wass_1(\mu^1_N,\mu^1) + \Wass_1(\mu^0_N,\mu^0) \notag\\
&\leq 2\eps_N \xrightarrow{N\to+\infty} \Wass_1(\mu^0,\mu^1) \label{eq:wassstab}
\end{align}

The last statement also follows from standard arguments from optimal
transport, which we summarize now. By \eqref{eq:lipschitz}--\eqref{eq:linfty}, we know that 
the sequences $(\psi^k_N)$ are uniformly bounded and uniformly Lipschitz. By Arzelà-Ascoli's theorem, 
this implies that the sequence $(\psi^k_N)$ admits a  subsequence converging uniformly
to some $\tilde{\psi}^k$. By \cite[Theorem 1.51]{santambrogio2015optimal},
$\tilde{\psi}^k$ is a Kantorovich
potential for the optimal transport problem between $\rho$ and $\mu^k$.
Since in addition,
\begin{equation} 0 = \lim_{N\to +\infty} \int \psi^k_N \dd\mu^k_N = \int
\tilde{\psi}^k \dd\mu^k, 
\end{equation}
we obtain, by Remark~\ref{rem:uniqueness}
that $\tilde{\psi}^k = \psi^k$. This shows that the whole sequence
$\psi^k_N$ converges uniformly to $\psi^k$. This implies as desired
\begin{equation*}
  \lim_{N\to +\infty} \int (\psi^1_N - \psi^0_N)^2 \dd(\mu^0_N + \mu^1_N) =
  \int (\psi^1 - \psi^0)^2 \dd(\mu^0 + \mu^1). \qedhere
\end{equation*}
\end{proof}

\subsection{Semi-discrete optimal transport}
In the remaining of this section, we work in the \emph{semi-discrete}
setting, assuming that all measures are supported on a (fixed) set
$\{y_1,\hdots,y_N\}$. Assuming that $\mu = \sum_{1\leq i\leq
  N}\mu_{i}\delta_{y_i}$, the Kantorovich dual to the optimal
transport problem between $\rho$ and $\mu$ problem can be written as
(e.g. \cite[Eq. (2.6)]{htter2019minimax}):
\newcommand{\Dual}{\mathrm{(D)}}
\newcommand{\Kant}{\mathcal{K}}
\begin{align} \label{eq:dual-sd}
    &\Dual = \min_{\psi} \int_{\X} \psi^* \dd \rho + \int_{\Y} \psi \dd\mu
\end{align}
where the minimum is taken among functions $\psi$ on $\{y_1,\hdots,y_N\}$. To simplify notations, we will often conflate the function $\psi$ with the vector $\bm{\psi}\in \Rsp^N$ defined by $\bm{\psi}_i = \psi(y_i)$. This vector $\bm{\psi}$ is also referred to as a (dual) \textit{potential} and defines a partition of the domain $\X$ into so-called Laguerre cells, described for all $1 \leq i \leq N$ by
\begin{equation}
    V_i(\bm{\psi}) = \{ x \in \X \mid \forall j,  \bm{\psi}_j \geq \bm{\psi}_i + \sca{y_j - y_i}{x}  \},
\end{equation}
so that
\begin{equation} \label{eq:dual-sdd}
  \Dual =  \min_{\bm{\psi} \in \Rsp^N}  \Kant(\bm{\psi})
\end{equation}
\begin{align}
  \hbox{ where } \Kant(\bm{\psi}) = \sum_{i=1}^N \int_{V_i(\bm{\psi})} (\sca{x}{y_i} - \bm{\psi}_i) \dd\rho(x) + \sum_{i=1}^N \mu_{i} \bm{\psi}_i.
\end{align}
By Theorem 1.1 in \cite{kitagawa:hal-01290496} (see also \cite{aurenhammer1998minkowski}),
\begin{equation} \nabla \Kant = G(\bm{\psi}) - \nu
\end{equation}
where
\begin{align}
&G_i(\bm{\psi})  = \rho(V_i(\bm{\psi}))\\
& G(\bm{\psi}) = (G_i(\bm{\psi}))_{1\leq i\leq N} \in\Rsp^N.
\end{align}
Therefore, a potential $\bm{\psi}$ solves problem (\ref{eq:dual-sdd}) if and
only if $G(\bm{\psi}) = \nu$.  The optimal potential $\bm{\psi}$ in
\eqref{eq:dual-sdd} then defines a Monge map $T: \X \rightarrow \Y$
that is piecewise constant, sending each point $x$ in $V_i(\bm{\psi})$ to
$y_i$. Alternatively, one can define $T = \nabla \phi$ where
$\phi=\psi^*$ is the Legendre transform of the function $\psi$ defined
by $\psi(y_i) = \bm{\psi}_i$.
Given a potential $\bm{\psi} \in \Rsp^N$, we denote
\begin{equation} \label{eq:mupsi}
\mu_{\bm{\psi}} = \sum_{1\leq i\leq N} G_i(\bm{\psi})\delta_{y_i}.
\end{equation}

\subsubsection*{Jacobian of $G$}
We consider the set $S_+\subseteq \Rsp^N$ of potentials such that all
Laguerre cells $V_i(\bm{\psi})$ contain some mass, defined by
\begin{equation}
S_+ = \{ \bm{\psi} \in\Rsp^N \mid \forall i, G_i(\bm{\psi}) > 0 \}.    
\end{equation}
From Theorems 1.3 and 4.1 in \cite{kitagawa:hal-01290496}, we know
that the map $G$ is $\Class^1$ on the set $S_+$. By Theorem~1.3 in
\cite{kitagawa:hal-01290496}, if $\bm{\psi}\in S_+$, the partial
derivatives of $G$ are given by
\begin{equation} \label{eq:DG}
     \begin{cases}
\frac{\partial G_i}{\partial \bm{\psi}_j}(\bm{\psi}) = \frac{\vol^{d-1}(V_i(\bm{\psi}) \cap V_j(\bm{\psi}))}{\nr{y_j-y_i}}  & \hbox{ for } i\neq j\\
\frac{\partial G_i}{\partial \bm{\psi}_i}(\bm{\psi}) = -\sum_{j\neq i} \frac{\partial G_i}{\partial \bm{\psi}_j}(\bm{\psi}) 
\end{cases} 
\end{equation}

\subsection{Proof of Theorem~\ref{th:stab-dual} (Stability of potentials) in the semi-discrete case}

In this section, we prove Theorem~\ref{th:stab-dual} when the measures
$\mu,\nu$ have the same finite support. This version of
Theorem~\ref{th:stab-dual} is rephrased as
Theorem~\ref{th:l2-potential-tv}.

\begin{theorem}
  \label{th:l2-potential-tv} Let $\bm{\psi}^0, \bm{\psi}^1$ be two potentials in $S_+$ satisfying
  \begin{equation} \label{eq:th:zeromean}
    \sca{\bm{\psi}^0}{G(\bm{\psi}^0)}_{\Rsp^N} = \sca{\bm{\psi}^1}{G(\bm{\psi}^1)}_{\Rsp^N} = 0.
  \end{equation}
  Then, with $\mu^k = \mu_{\bm{\psi}^k},$
  \begin{equation} \label{eq:thsd:stab-tv}
    \sca{(\bm{\psi}^1 - \bm{\psi}^0)^2}{G(\bm{\psi}^0)+G(\bm{\psi}^1)}_{\Rsp^N} \leq C \nr{\mu^1 - \mu^0}_{\TV}. \end{equation}
\begin{equation}\label{eq:stab-dual:w}
\sca{(\bm{\psi}^1 - \bm{\psi}^0)^2}{G(\bm{\psi}^0)+G(\bm{\psi}^1)}_{\Rsp^N} \leq C 
\Wass_1(\mu^1, \mu^0)^{\frac{2}{3}}, \end{equation}
where $C$ depends only on $d,\X$ and $\Y$.
\end{theorem}

We will require two preliminary
results. The next lemma follows from Brunn-Minkowski's inequality. This inequality has already appeared in the numerical analysis of Monge-Ampère
equations, see \cite{benamou2016discretization,nochetto2019pointwise}.

\begin{lemma} \label{lemma:tv} Let $\bm{\psi}^0,\bm{\psi}^1\in S^+$ and consider 
$\bm{\psi}^t = (1-t) \bm{\psi}^0 + t\bm{\psi}^1$. Then, 
\begin{equation}\label{eq:lemma:tv1}
\forall i,~ G_i(\bm{\psi}^t)^{\frac{1}{d}} \geq (1-t) G_i(\bm{\psi}^0)^{\frac{1}{d}} + tG_i(\bm{\psi}^1)^{\frac{1}{d}}
\end{equation}
In particular, $\bm{\psi}^t\in S^+$. Moreover,
\begin{equation} \label{eq:lemma:tv2}
\nr{G(\bm{\psi}^t) - G(\bm{\psi}^0)}_1 \leq \nr{G(\bm{\psi}^1) - G(\bm{\psi}^0)}_1, 
\end{equation}
\begin{equation}\label{eq:lemma:tv3}
 \nr{G(\bm{\psi}^t) - G(\bm{\psi}^0)}_1 \leq 2(1 - (1-t)^{d}). 
\end{equation}
\end{lemma}

\begin{proof}
   Let $x^0 \in V_i(\bm{\psi}^0)$ and $x^1 \in V_i(\bm{\psi}^1)$. Then, for all $j\in \{1, \dots, N \}$,
\begin{equation} \begin{cases}
\bm{\psi}^0_j\geq \bm{\psi}^0_i + \sca{y_j-y_i}{x^0}\\
\bm{\psi}^1_j\geq \bm{\psi}^1_i + \sca{y_j-y_i}{x^1}.\\
\end{cases} 
\end{equation}
Taking the convex combination of these inequalities we get for all $j\in \{1, \dots, N \}$,
\begin{equation}
     \bm{\psi}^t_j\geq \bm{\psi}^t_i + \sca{y_j-y_i}{(1-t)x^0+t x^1}.
\end{equation}
This shows that $(1-t)x^0 + t x^1 \in V_i(\bm{\psi}^t)$ (note that we use the convexity of $\X$ here). Thus, 
\begin{equation}
(1-t) V_i(\bm{\psi}^0) + t V_i(\bm{\psi}^1) \subseteq V_i(\bm{\psi}^t). 
\end{equation}
Taking the Lebesgue measure on both sides and applying Brunn-Minkowski's inequality gives
\begin{equation}
\begin{aligned}
    G_i(\bm{\psi}^t)^{1/d} =  \rho(V_i(\bm{\psi}^t))^{1/d} 
    &\geq  \rho((1-t) V_i(\bm{\psi}^0) + t V_i(\bm{\psi}^1))^{1/d} \\
    &\geq (1-t) \rho(V_i(\bm{\psi}^0))^{1/d} + t \rho(V_i(\bm{\psi}^1))^{1/d}  \\
    &\geq (1-t)G_i(\bm{\psi}^0)^{1/d} + t G_i(\bm{\psi}^1))^{1/d}
\end{aligned}
\end{equation}
This inequality directly implies 
\begin{equation}
    \begin{aligned}
    &G_i(\bm{\psi}^t) \geq \min(G_i(\bm{\psi}^0), G_i(\bm{\psi}^1)),\\
&\hbox{i.e.} \min(G_i(\bm{\psi}^t), G_i(\bm{\psi}^0)) \geq \min(G_i(\bm{\psi}^0),G_i(\bm{\psi}^1)).
    \end{aligned}
\end{equation} 
Using the following equivalent formulation of the TV distance between probability measures we get \eqref{eq:lemma:tv2}:
\begin{equation}
\begin{aligned} 
\frac{1}{2} \nr{G(\bm{\psi}^t) - G(\bm{\psi}^0)}_1 &= 1 - \sum_i \min(G_i(\bm{\psi}^t), G_i(\bm{\psi}^0))  \\
&\leq 1 - \sum_i \min(G_i(\bm{\psi}^0), G_i(\bm{\psi}^1)) = \frac{1}{2} \nr{G(\bm{\psi}^t) - G(\bm{\psi}^0)}_1. 
\end{aligned}
\end{equation}
To prove \eqref{eq:lemma:tv3}, we first remark that by \eqref{eq:lemma:tv1},
\begin{equation}
    \begin{aligned}
    &G_i(\bm{\psi}^t) \geq (1-t)^d G_i(\bm{\psi}^0), \\
    &\hbox{i.e.} \min(G_i(\bm{\psi}^t), G_i(\bm{\psi}^0)) \geq (1-t)^d G_i(\bm{\psi}^0).
    \end{aligned}
\end{equation}
We conclude using the same formula as above:
\begin{equation}
\begin{aligned} 
\frac{1}{2} \nr{G(\bm{\psi}^t) - G(\bm{\psi}^0)}_1 &= 1 - \sum_i \min(G_i(\bm{\psi}^t), G_i(\bm{\psi}^0))  \\
&\leq 1 - \sum_i (1-t)^d G_i(\bm{\psi}^0) = 1 - (1-t)^d. 
\end{aligned}
\end{equation}
\end{proof}

The next proposition gives an explicit lower bound on the smallest
non-zero eigenvalue of the opposite of Jacobian matrix of the map
$G$. Its proof follows from the stability analysis of finite volumes
discretization of elliptic PDEs, see Lemma~3.7 in
\cite{eymard2000finite}. We report this proof (with very minor
adaptations to our case) in the appendix.

\begin{proposition}[Discrete Poincaré-Wirtinger inequality] \label{prop:poincare}
Consider $\bm{\psi}\in S_+$ and $v\in \Rsp^N$.
Then,
\begin{equation}
\sca{v^2}{G(\bm{\psi})}_{\Rsp^N} - \sca{v}{G(\bm{\psi})}^2_{\Rsp^N} \leq - C \sca{\DD G(\bm{\psi}) v}{v}_{\Rsp^N}    
\end{equation} 
where $C = C(d) \diam(\Y) \diam(\X)^{d+1} > 0$.
\end{proposition}
\begin{remark}
  In particular, $-\DD G(\bm{\psi})$ is semidefinite positive, since its
  smallest non-zero eigenvalue is greater than a variance. This can
  also be seen from the definition of $\DD G(\bm{\psi})$ as a Laplacian
  matrix, or simply from Gershgorin's circle theorem and the explicit
  formula for $\DD G(\bm{\psi})$ recalled in \eqref{eq:DG}.
\end{remark} 

With these two results at hand, we show $L^2$ stability of the dual
potentials in the semi-discrete case. 

\begin{proof}[Proof of Theorem~\ref{th:l2-potential-tv}]
In this proof,  $A\lesssim B$ means that $A \leq C B$ for a constant $C$  depending only on  $d$, the diameters of $\X$ and $\Y$, $M_\X$ and $M_\Y$.
Denote $\bm{\psi}^t = (1-t) \bm{\psi}^0 + t\bm{\psi}^1$ and $v = \bm{\psi}^1 - \bm{\psi}^0$. By Taylor's formula, 
\begin{equation} \label{eq:taylor}
  \sca{G(\bm{\psi}^1) - G(\bm{\psi}^0)}{v}_{\Rsp^N} = \int_0^1 \sca{\DD G(\bm{\psi}^t) v}{v}_{\Rsp^N} \dd t
\end{equation}
Moreover, Proposition~\ref{prop:poincare} gives
\begin{equation} \label{eq:poincare} \sca{v^2}{G(\bm{\psi}^t)}_{\Rsp^N} - \sca{v}{G(\bm{\psi}^t)}^2_{\Rsp^N} \lesssim
  - \sca{\DD G(\bm{\psi}^t) v}{v}_{\Rsp^N} 
\end{equation}
Let us restrict to $t \in [0,\frac{1}{4}]$. Then, by Eq.~\eqref{eq:lemma:tv1}, one has 
\begin{equation} 
  G_i(\bm{\psi}^t) \geq (1-t)^{d} G_i(\bm{\psi}^0) \gtrsim G_i(\bm{\psi}^0),
\end{equation}
Thus, on the interval $t\in [0,\frac{1}{4}]$, 
\begin{equation} \label{eq:sqnorm}
  \sca{v^2}{G(\bm{\psi}^0)} \lesssim \sca{v^2}{G(\bm{\psi}^t)}.
  \end{equation}
On the other hand, using the assumption \eqref{eq:th:zeromean}, we get
\begin{align}
  \sca{v}{G(\bm{\psi}^t)}_{\Rsp^N} &= \sca{\bm{\psi}^1 - \bm{\psi}^0}{G(\bm{\psi}^t)}_{\Rsp^N} \notag\\
  &= \sca{\bm{\psi}^1 - \bm{\psi}^0}{G(\bm{\psi}^t) - G(\bm{\psi}^0)}_{\Rsp^N} + \sca{\bm{\psi}^1 - \bm{\psi}^0}{G(\bm{\psi}^0)}_{\Rsp^N} \notag\\
  &= \sca{\bm{\psi}^1 - \bm{\psi}^0}{G(\bm{\psi}^t) - G(\bm{\psi}^0)}_{\Rsp^N} + \sca{\bm{\psi}^1}{G(\bm{\psi}^0)-G(\bm{\psi}^1)}_{\Rsp^N}
\end{align}
thus implying
\begin{align}
  \abs{\sca{v}{G(\bm{\psi}^t)}} &\leq \nr{\bm{\psi}_1-\bm{\psi}_0}_\infty  \nr{G(\bm{\psi}^t) - G(\bm{\psi}^0)}_1 +
  \Lip(\bm{\psi}_1) \Wass_1(\mu_0,\mu_1)  \notag \\ 
&\lesssim \nr{G(\bm{\psi}^t) - G(\bm{\psi}^0)}_1 + \Wass_1(\mu_0,\mu_1) \label{eq:zm}
\end{align}
where we used Kantorovich-Rubinstein's theorem to get the first
inequality and that $\nr{\bm{\psi}_1-\bm{\psi}_0}_\infty$ is bounded by a
constant depending on $\X$ and $\Y$, as in \eqref{eq:linfty}.
Using Kantorovich-Rubinstein's theorem again, we also get
\begin{equation} \label{eq:kragain}
\begin{aligned}
\abs{\sca{G(\bm{\psi}^1) - G(\bm{\psi}^0)}{v}} 
&\lesssim \Wass_1(\mu^0, \mu^1) \\
&\lesssim \nr{\mu^0 - \mu^1}_\TV.
\end{aligned}
\end{equation}
Proposition~\ref{prop:poincare} implies that $\sca{\DD G(\bm{\psi}^t)v}{v} \leq 0$ for all $t\in [0,1]$. Combining \eqref{eq:taylor},\eqref{eq:poincare},\eqref{eq:sqnorm}, \eqref{eq:zm} and \eqref{eq:kragain} gives us
\begin{equation} \label{eq:stab}
\begin{aligned}
&T \sca{v^2}{G(\bm{\psi}^0)}   
\lesssim \Wass_1(\mu^0,\mu^1) + \int_0^T  (\nr{G(\bm{\psi}^t) - G(\bm{\psi}^0)}_1+\Wass_1(\mu_0,\mu_1))^2 \dd t
\end{aligned}
\end{equation}
To conclude the proof of the stability  with respect to the total
variation norm \eqref{eq:thsd:stab-tv}, we simply note that thanks to Lemma~\ref{lemma:tv} \eqref{eq:lemma:tv2}, we have $\nr{G(\bm{\psi}^t)-G(\bm{\psi}^0)}_1 \leq \nr{G(\bm{\psi}^1)-G(\bm{\psi}^0)}_1$. Combining
with the comparison $\Wass_1 \lesssim \nr{\cdot}_\TV$, \eqref{eq:stab} with 
$T = \frac{1}{4}$ yields
\begin{equation}
    \sca{v^2}{G(\bm{\psi}^0)} \lesssim \nr{\mu^1 - \mu^0}_\TV.
\end{equation}  
Note that thanks to symmetry, we get the same upper bound with $G(\bm{\psi}^0)$ replaced
by $G(\bm{\psi}^1)$. Summing these bounds thus concludes the proof of \eqref{eq:thsd:stab-tv}.

To get the second stability inequality \eqref{eq:stab-dual:w}, with
respect to the Wasserstein distance, we use
Lemma~\ref{lemma:tv}--\eqref{eq:lemma:tv3}, which gives for
$t\in[0,T],$
\begin{equation}
\nr{G(\bm{\psi}^t) -  G(\bm{\psi}^0)}_1 \leq 2(1 - (1 - t)^d) \lesssim T. 
\end{equation} 
Combining this inequality with Eq.~\eqref{eq:stab} we get for $T\leq \frac{1}{4}$,
\begin{equation}
     T \sca{v^2}{G(\bm{\psi}^0)} \lesssim \Wass_1(\mu^0,\mu^1) + T(T+\Wass_1(\mu_0,\mu_1))^2. 
\end{equation}
If $\Wass_1(\mu^0,\mu^1)^{\frac{1}{3}} \leq \frac{1}{4}$,
we can choose $T = \Wass_1(\mu^0,\mu^1)^{\frac{1}{3}}$ to obtain the desired inequality \eqref{eq:stab-dual:w}. On the other hand, if $\Wass_1(\mu^0,\mu^1)^{\frac{1}{3}} \geq \frac{1}{4}$,  taking $T = \frac{1}{4}$ gives us
 \begin{equation}
\sca{v^2}{G(\bm{\psi}^0)} 
\lesssim \Wass_1(\mu^0,\mu^1) 
= D \frac{\Wass_1(\mu,\nu)}{D} \leq D  \left(\frac{\Wass_1(\mu,\nu)}{D}\right)^{2/3}
\end{equation}
with $D := \max_{\mu,\nu\in\Prob(\Y)} \Wass_1(\mu,\nu) 
\leq \diam(\Y)$
thus also proving \eqref{eq:stab-dual:w} in that case.
\end{proof} 

\subsection{Proof of Theorem~\ref{th:map-holder-tv} (Stability of optimal transport maps)}
We need a result from \cite{chazal:hal-01425558}, providing an upper bounds the $L^2$ norm between gradients of  convex functions.
\begin{proposition}[\cite{chazal:hal-01425558} Theorem 22]\label{prop:chazal}
Let $f$ and $g$ be convex functions on a bounded convex set $\X$, then 
\begin{equation}
    \|\nabla f - \nabla g\|_{\L^2} \leq 2 C_{\X} \|f-g\|_\infty^{1/2} (\|\nabla f\|_\infty^{1/2} + \|\nabla g\|_\infty^{1/2})
\end{equation}where $C_{\X}$ depends only on $\X$.
\end{proposition}
The stability of potentials (Theorem \ref{th:l2-potential-tv}) implies that 
\begin{equation}\label{eq:l2}
\begin{aligned}
    &\|\psi^0 - \psi^1\|_{\L^2(\mu^0+\mu^1)}^2 \lesssim \eps \\
    &\hbox{ with } \eps = \|\mu^0 - \mu^1\|_\TV \hbox{ or } \eps = \Wass_1(\mu^0,\mu^1)^{\frac{2}{3}}.
\end{aligned}
\end{equation}
In practice, these $L^2$ estimates are not sufficient to conclude, and we need to translate them into a $L^\infty$ estimate in order to apply Proposition~\ref{prop:chazal}. For this purpose,
we consider  $\alpha\in(0,1)$, and we define
\begin{equation} \label{eq:Yalpha}
\Y_\alpha = \{ y \in \Y \mid \abs{\psi^0(y) - \psi^1(y)}\leq \eps^\alpha \}. 
\end{equation}
By Chebyshev's inequality, we deduce from \eqref{eq:l2} that  for $k\in\{0,1\},$
\begin{equation}
\eps^{2\alpha} \mu^k(\Y \setminus \Y_\alpha) \leq \nr{\psi^0 - \psi^1}^2_{\L^2(\mu^k)} \lesssim \eps, 
\end{equation}
which gives
\begin{equation} 1 - \mu^k(\Y_\alpha) \lesssim \eps^{1-2\alpha}. 
\end{equation}
We construct the Legendre
transform of the functions $\psi^k$ on the whole set $\Y$, and 
of the restrictions of $\psi^k$ to the set $\Y_\alpha$:
\begin{align}
    \label{eq:leg}
 \phi^k(x) &= \max_{y\in \Y} \sca{x}{y} - \psi^k(y),\\
 \label{eq:legalpha}
\phi^{k,\alpha}(x) &= \max_{y\in \Y_\alpha} \sca{x}{y} - \psi^k(y).    
\end{align}
Comparing Eqs. \eqref{eq:leg} and \eqref{eq:legalpha}, one sees that $\phi^{k,\alpha}\leq \phi^k$. 
Moreover, if $\nabla\phi^k(x) \in \Y_\alpha$, then using the Fenchel-Young (in)equality,
\begin{equation}
    \phi^k(x) + \psi^k(\nabla\psi^k(x)) = \sca{x}{\nabla \psi^k(x)}
\leq \phi^{k,\alpha}(x) + \psi^{k,\alpha}(\nabla \psi^k(x)),
\end{equation} so that $\phi^k(x) =
\phi^{k,\alpha}(x)$. In other words, $\phi^{k,\alpha} \equiv \phi^{k}$
on the set
\begin{equation}
  \X_\alpha = (\nabla \phi^k)^{-1}(\Y_\alpha).
\end{equation}
Note also that this set $\X_\alpha^k$ is "large", in the sense that
\begin{equation}
\begin{aligned}
1-\rho(\X_\alpha^k) &= 1-\rho((\nabla \phi^k)^{-1}(\Y_\alpha)) \\
&=  1-\mu^k(\Y_\alpha) \lesssim \eps^{1-2\alpha}, 
\end{aligned}
\end{equation}
where we used $\nabla \phi^k_{\#} \rho = \mu^k$.
As in \eqref{eq:lipschitz}, the gradients $\nabla \phi^{k,\alpha}$ and $\nabla \phi^{k}$ are
uniformly bounded by $M_\Y$ (by Eqs. \eqref{eq:leg} and \eqref{eq:legalpha}) and they coincide on the "large" set $\X_\alpha^k$. This directly implies that they are close in $L^2$ norm:
\begin{equation}\label{eq:estim1}
\begin{aligned}
\|\nabla \phi^{k,\alpha} - \nabla \phi^k\|_{\L^2(\X)} &= \|\nabla \phi^{k,\alpha} - \nabla \phi^k\|_{\L^2(\X\setminus \X^k_\alpha)}\\
&\quad \leq  (1 - \rho(\X^k_\alpha))(\|\nabla \phi^{k,\alpha}\|_\infty + \|\nabla \phi^k\|_\infty) 
\lesssim \eps^{1-2\alpha}. 
\end{aligned}
\end{equation}
On the other hand, by definition of $\Y_\alpha$ (see Eq. \eqref{eq:Yalpha}), the functions
$\psi^0$ and $\psi^1$ are uniformly close on  the set $\Y_\alpha.$ This implies that the  
Legendre transforms $\phi^{0,\alpha}$ and $\phi^{1,\alpha}$, defined in \eqref{eq:legalpha},
are also close. Indeed,
\begin{equation}
\begin{aligned}
    \phi^{0,\alpha}(x) &= \max_{y\in \Y^\alpha} \sca{x}{y} - \psi^0(x) \\
              &\leq \max_{y\in \Y^\alpha} \sca{x}{y} - \psi^1(x) + \eps^\alpha \\
              &= \phi^{1,\alpha}(x) + \eps^\alpha,
\end{aligned}
\end{equation}
thus giving by symmetry
\begin{equation}
 \|\phi^{1,\alpha} - \phi^{0,\alpha}\|_\infty \leq \eps^\alpha. 
\end{equation}
Combining this inequality with Proposition \ref{prop:chazal}, we obtain
\begin{equation} \label{eq:estim2}
\begin{aligned}
\|\nabla \phi^{1,\alpha} - \nabla \phi^{0,\alpha}\|_{\L^2(\X)}
\lesssim 2 (\|\nabla \phi^{0,\alpha}\|_\infty + \|\nabla \phi^{1,\alpha}\|_\infty)^{1/2}
\|\phi^{1,\alpha} - \phi^{0,\alpha}\|_\infty^{1/2} 
\lesssim \eps^{\frac{\alpha}{2}} 
\end{aligned}
\end{equation}
Using the triangle inequality and the two previous estimations \eqref{eq:estim1}--\eqref{eq:estim2}, we obtain
\begin{align}
&\|\nabla\phi^1 - \nabla\phi^0\|_{\L^2(\X)} \notag\\
&\leq \|\nabla\phi^1 - \nabla\phi^{1,\alpha}\|_{\L^2(\X)} + 
\|\nabla\phi^{1,\alpha} - \nabla\phi^{0,\alpha}\|_{\L^2(\X)} + 
\|\nabla \phi^{0,\alpha} - \nabla\phi^{0}\|_{\L^2(\X)} \notag\\
&\lesssim \eps^{1-2\alpha} + \eps^{\alpha/2}
\end{align}
The best exponent is obtained when $1-2\alpha = \alpha/2$ i.e. $1 = 5\alpha/2$, $\alpha = 2/5$, giving
\begin{equation}\|\nabla \phi^1 - \nabla\phi^0\|_{\L^2(\X)} \lesssim \eps^{\frac{1}{5}}, \end{equation}
which implies the desired estimates if one replaces
$\eps$ with the  possible values \eqref{eq:l2}.

\section{Experiments}

\begin{figure*}
    \centering
    \includegraphics[width=.32\textwidth]{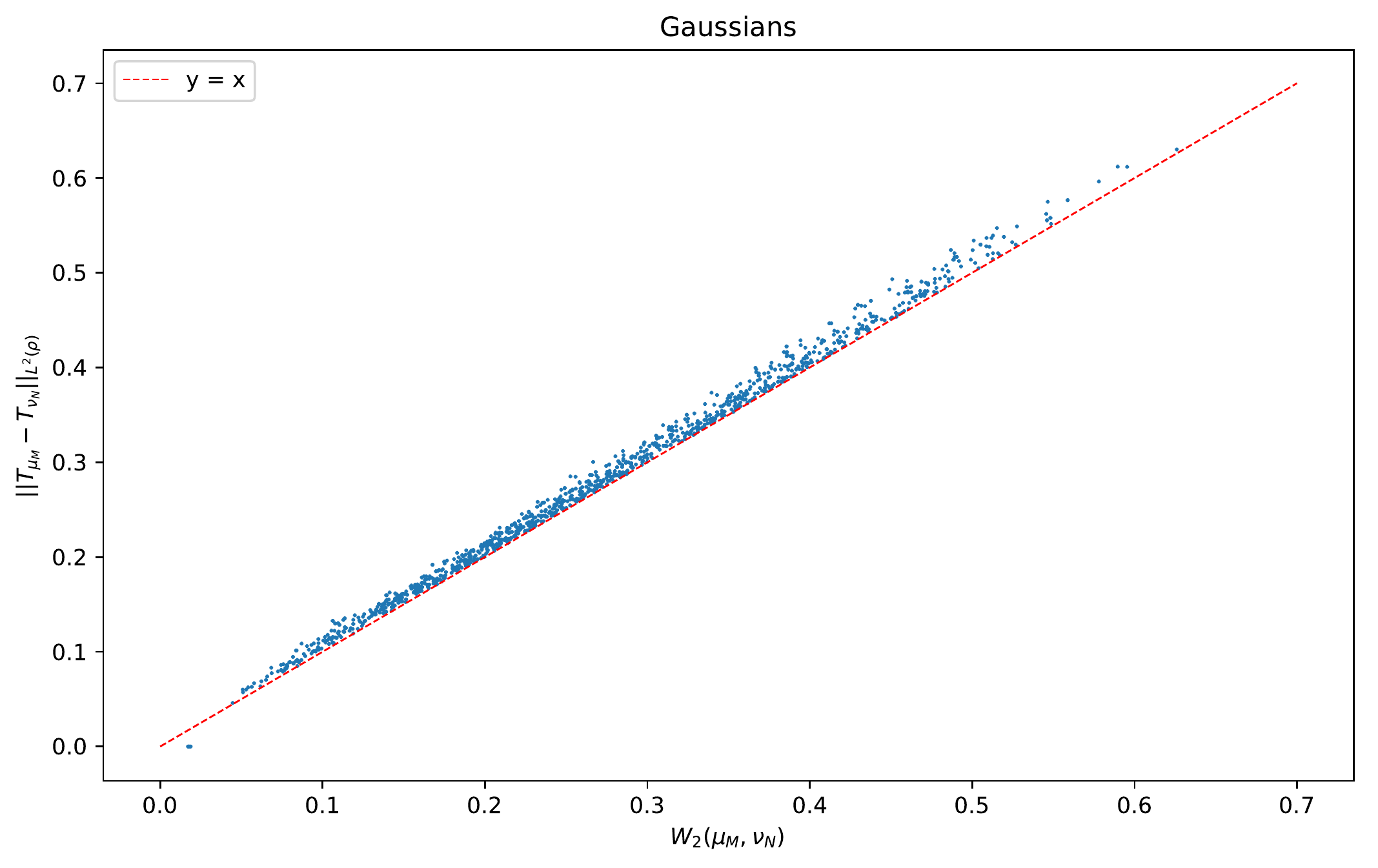}  \hfill
    \includegraphics[width=.32\textwidth]{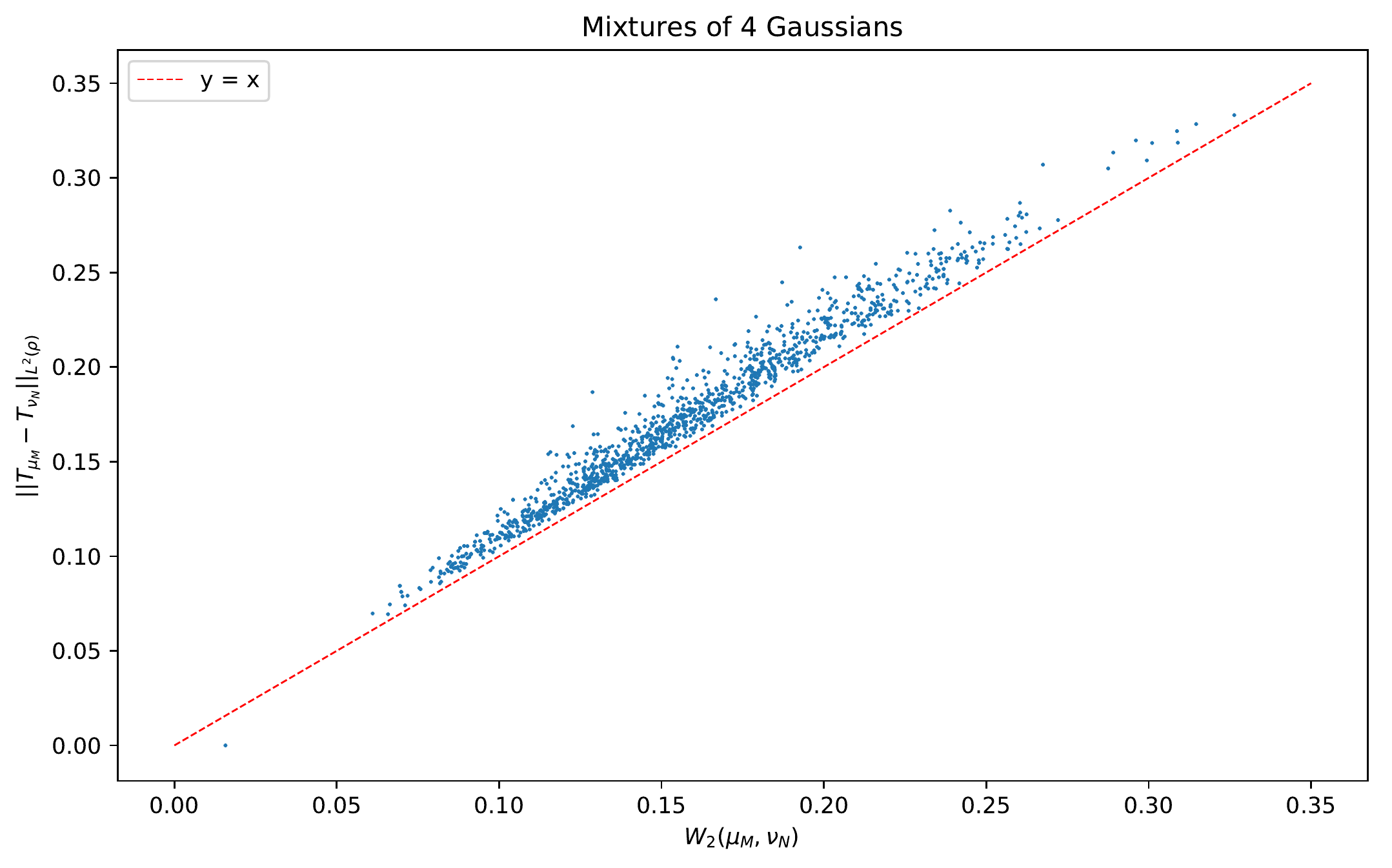} \hfill
    \includegraphics[width=.32\textwidth]{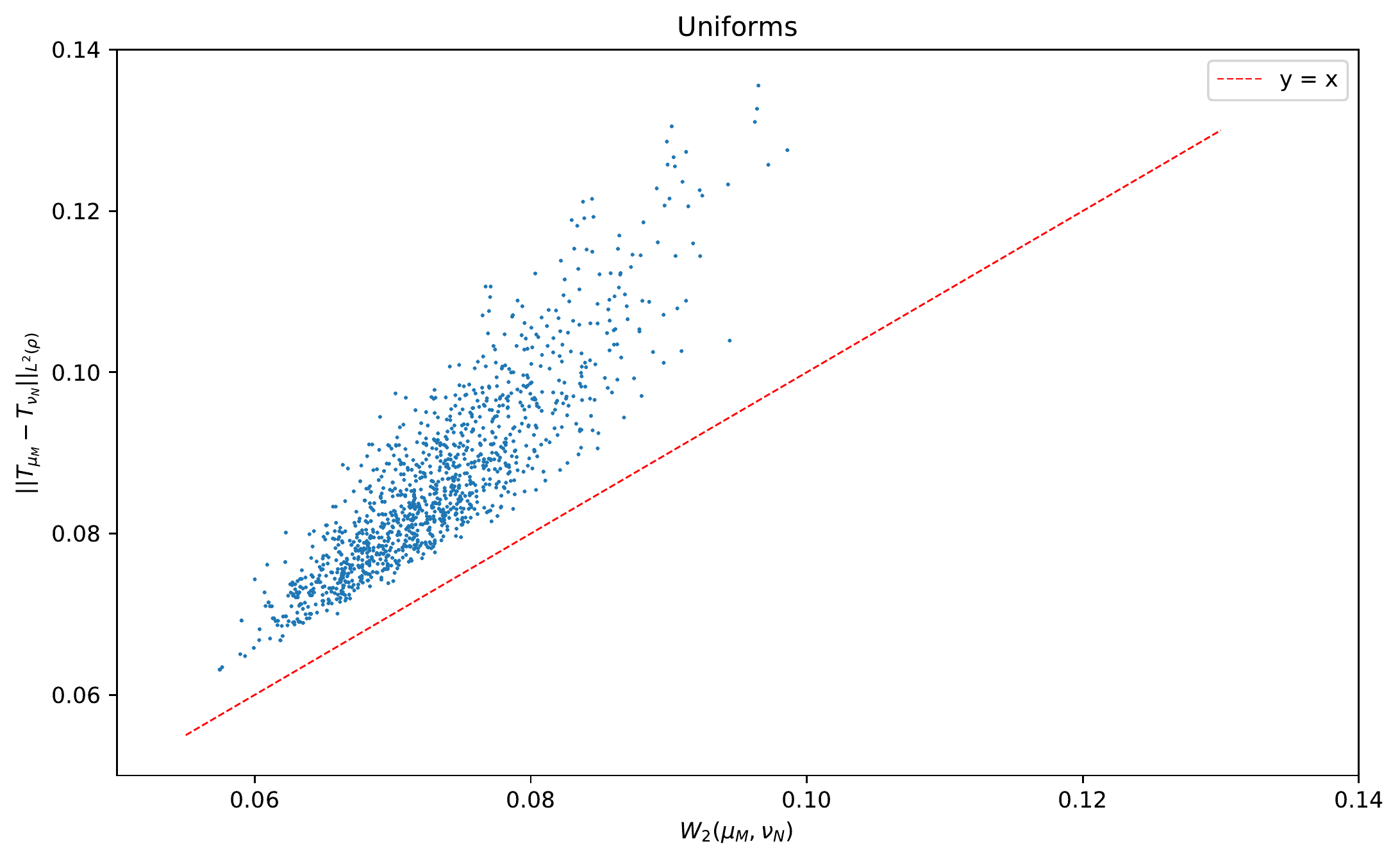}
    \caption{$W_{2,\rho}$ vs. $\Wass_2$ between point clouds sampled from Gaussian, Mixture of 4 Gaussian and Uniform distributions. $\Wass_2$ being approximated with entropic regularization, we may have $\Wass_2 \geq W_{2,\rho}$ on certain points.}
    \label{fig:distances}
\end{figure*}

\begin{figure*}
    \centering
    \includegraphics[width=.60\textwidth]{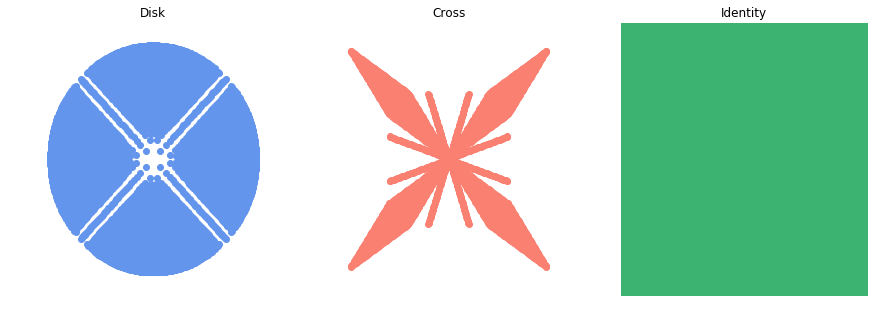}  
    \hspace{3px}
    \includegraphics[width=.35\textwidth]{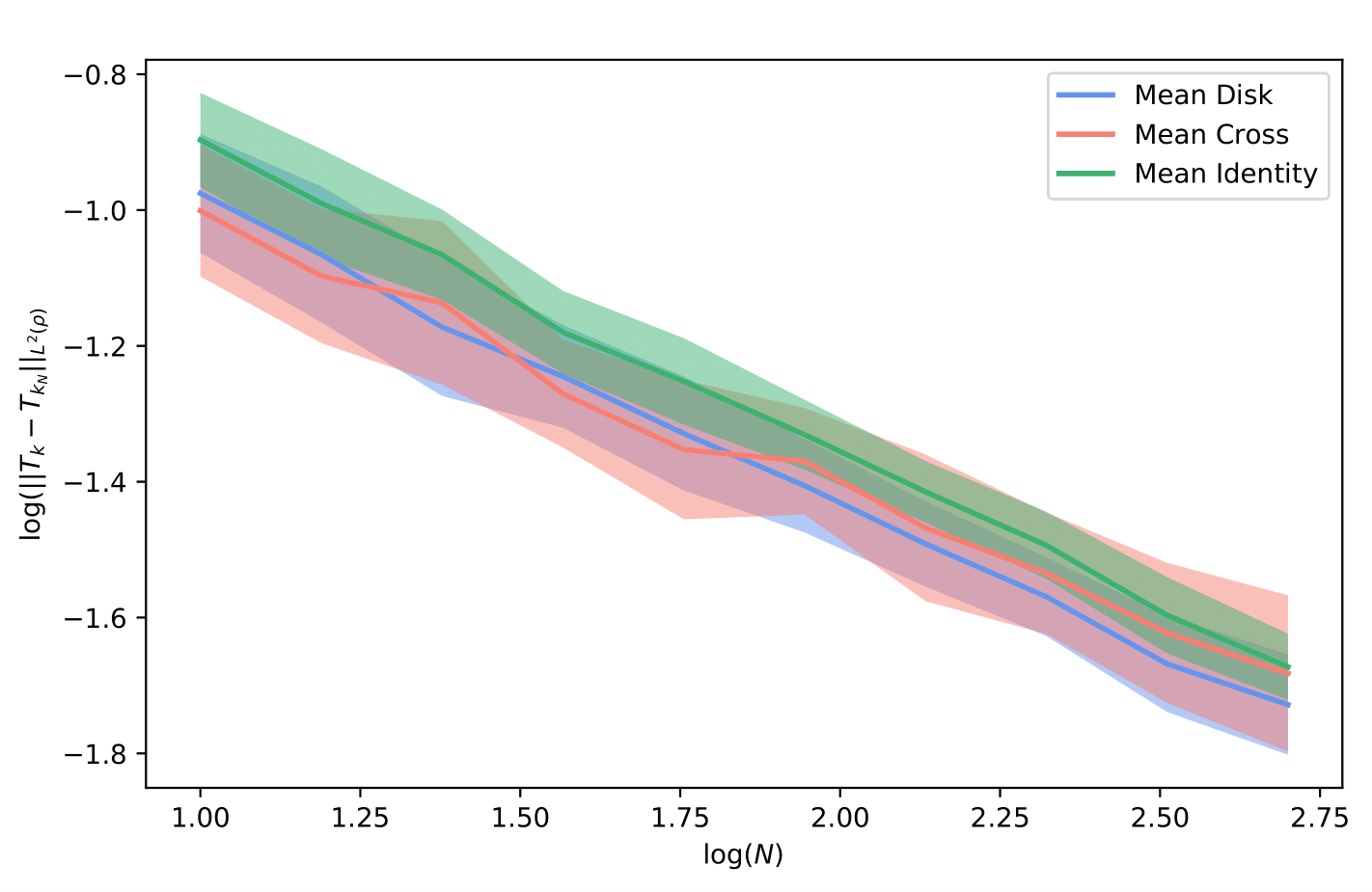}
    \caption{(Left) Target measures, push-forwards of the maps $T_k = \nabla \phi_k$ where $\phi_{\text{Disk}}(x, y) := 0.25(x+y) + 0.07(|x+y|^{3/2} + |x-y|^{3/2})$, $\phi_{\text{Cross}}(x, y) := 0.5(x+y) + 0.04 \max(4(x+y-1)^2 + 0.5(2x-1)^2 + 0.5(2y-1)^2, 4(x-y)^2 + 0.5(2x-1)^2 + 0.5*(2y-1)^2)$ and $\phi_{\text{Square}}(x, y) :=  0.5(x^2 + y^2)$ (Right) Sampling distance $\|T_\mu - T_{\mu_N}\|_{\L^2(\rho)}$.}
    \label{fig:sampling}
\end{figure*}

In this last section, we briefly illustrate the behaviour of the Monge
map embeddings and we mention potential use of these embeddings in
machine learning. In what follows, we consider that $d = 2$ and that
$\rho$ is the Lebesgue measure on the unit square $\X = [0,
  1]^2$. For simplicity, the discrete measures $\mu$ and $\nu$ are
also supported on $\Y = \X$, for which algorithms
readily give approximates of $\Wass_p(\mu, \nu)$ and of $T_\mu$ or
$T_\nu$. The Wasserstein distance $\Wass_p(\mu, \nu)$ is approximated
using Sinkhorn's algorithm \cite{Cuturi:2013} while $T_\mu$ and
$T_\nu$ are approximated with a damped Newton's algorithm
\cite{kitagawa:hal-01290496}.

\subsection{Vectorization of the Monge maps}
The Hilbert space $H=\L^2(\rho,\Rsp^2)$, which contains the Monge maps
$T_\mu,T_\nu$ is infinite dimensional. We therefore project the maps
on the finite dimensional subspace $H_m \subseteq H$ of piecewise constant maps, defined for any
$m\in \Nsp$ by 
\begin{equation}
  H_m = \{ T \in \L^2(\rho,\Rsp^2) \mid \forall s,t\in\{0,\hdots, m-1\}, \restr{T}_{\X_{s,t}} \equiv \hbox{ cst}\},
  \end{equation}
where $\X_{s,t} = [\frac{s}{m}, \frac{s+1}{m}) \times [\frac{t}{m}, \frac{t+1}{m})$. The orthogonal projection $\Pi_m: H \to H_m$ can be computed using
    \begin{equation}
      \restr{\Pi_m T}_{\X_{s,t}} = m^2 \int_{\X_{s,t}} T.
      \end{equation}
As the projection on a close subspace, the  mapping $\Pi_m$ is
$1$-Lipschitz. This implies that the vectorized Monge embedding $\mu \mapsto \Pi_m\circ T_\mu$ satisfies the same Hölder-continuity results as the Monge embedding since
\begin{equation}
  \|\Pi_m T_\mu - \Pi_m T_\nu\|_{\L^2(\rho)} \leq \|T_\mu - T_\nu\|_{\L^2(\rho)}.
  \end{equation}
In practice,  $\Pi_m T_\mu$ is represented by the $m^2d$-dimensional vector
\begin{equation}
  \bm{T_\mu} := \left(\int_{\X_{s,t}} T_\mu \dd\rho\right)_{1\leq s, t \leq m}.
\end{equation}

\subsection{Distance approximation}
We first compare $W_{2,\rho}(\mu, \nu) = \|T_\mu - T_\nu\|_{\L^2(\rho)}$ against $\Wass_2(\mu, \nu)$ in specific settings to illustrate Equation (\ref{eq:biHolder-dist_W2rho}). We consider three different settings corresponding to three different families of distributions. In each setting, $50$ point clouds of $300$ points are sampled, each from a random distribution that belongs to the given family, and pairwise $\Wass_2$ and $W_{2,\rho}$ distances on the $50$ point clouds are computed. $\Wass_2$ is approximated with Sinkhorn's algorithm while the transport maps $T_\mu$ are approximated using \cite{kitagawa:hal-01290496}. The distances $\|T_\mu - T_\nu\|_{\L^2(\rho)}$ are approximated with $ \|\bm{T_\mu} - \bm{T_\nu}\|_2$ with $m = 200$.

The three families of distributions we consider are: Gaussian, Mixture of $4$ Gaussians and Uniform. Note that for each point cloud sampling in the two first settings the parameters of the sampled distribution are selected randomly. We report in Figure \ref{fig:distances} the comparisons between $W_{2,\rho}$ and $\Wass_2$. We observe that $W_{2,\rho}$ behaves like $\Wass_2$ when the target measure are concentrated (Gaussian and Mixture of Gaussians distributions) and that this proximity of the two distances fades when the target measures have less concentrated or are drawn from the same distribution (Uniform).

\subsection{Sampling approximation}

In practice, the population distribution $\mu$ is often unknown and one can only access to samples $(x_i)_{i=1, \dots, N}$ from this distribution, yielding the empirical distribution $\mu_N = \frac{1}{N} \sum_{i=1}^N \delta_{x_i}$. One can thus wonder how well $T_{\mu_N}$ represents $T_{\mu}$ in function of the number of samples $N$. We illustrate the sampling approximation of $T_{\mu_N}$ by observing the quantity $\|T_\mu - T_{\mu_N}\|_{\L^2(\rho)}$ as a function of $N$ in again $3$ different settings where the "ground truth map" $T_\mu$ is prescribed. The $3$ maps are chosen as gradients of convex functions and transport the unit square to measures resembling a disk, a cross and a square (Figure \ref{fig:sampling}). For the different values of $N$ the experiments are repeated $25$ times and the standard deviations define the shaded areas surrounding the curves. We can observe a slightly better sampling behavior for the identity map (defining the square measure), which might be due to the regularity of the transport map.

In a more statistical context, we observe in Figure \ref{fig:sampling-2} the same quantities when the target measures are a Gaussian, a Mixture of $4$ Gaussians and the uniform distribution on $\X$. Since the "ground truth" maps $T_\mu$ are unknown in these case, we approximate them with the map $T_{\mu_M}$ for $M=10000$. Again, the measures that have the most concentrated support seem to have a better sampling behavior.

\begin{figure*}   
    \centering
    \includegraphics[width=.20\textwidth]{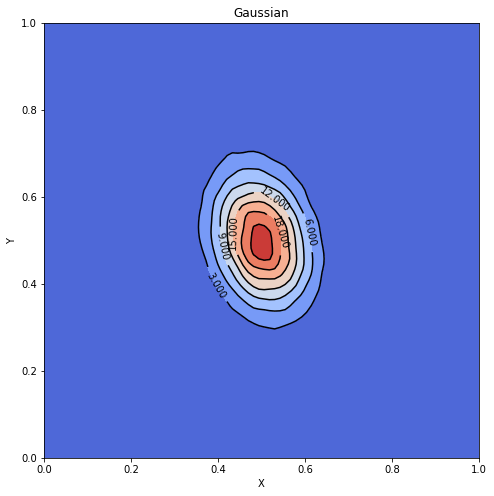}  
    \hspace{3px}
    \includegraphics[width=.20\textwidth]{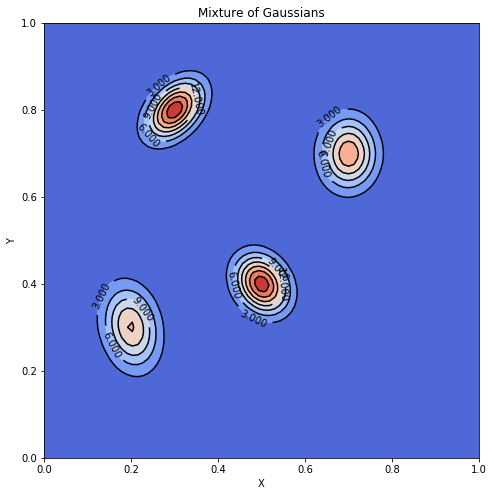}
    \hspace{3px}
    \includegraphics[width=.20\textwidth]{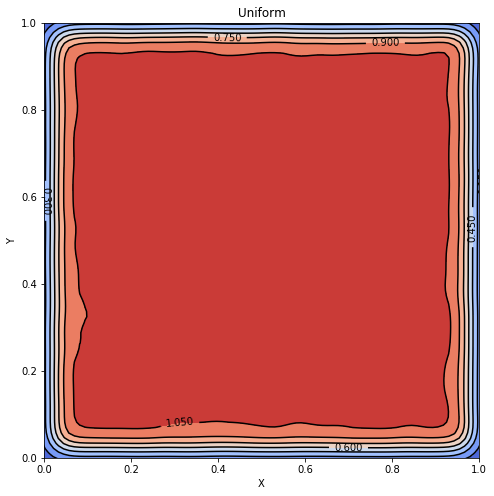}
    \hspace{3px}
    \includegraphics[width=.30\textwidth]{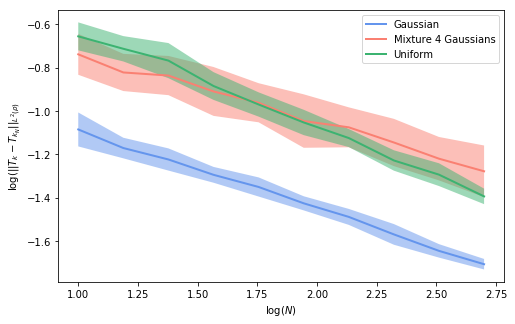}
    \caption{From left to right: densities of the sampled Gaussian, Mixture of $4$ Gaussians and Uniform distributions and sampling distance $\|T_\mu - T_{\mu_N}\|_{\L^2(\rho)}$ as a function of $N$}
    \label{fig:sampling-2}
\end{figure*}

\begin{figure}   
    \centering
    \includegraphics[width=.5\textwidth]{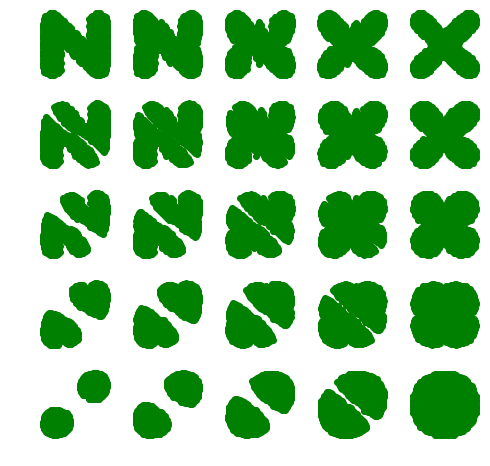}  
    \caption{Barycenters of $4$ point clouds. Weights $(\lambda_s)_s$ are bilinear w.r.t the corners of the square.}
    \label{fig:barycenter}
\end{figure}

\subsection{Barycenter approximation and clustering}
Computing means and barycenters is often necessary in unsupervised learning contexts. For point cloud data, the Wasserstein distance is a natural choice to define such barycenters. For $(\mu_s)_{s=1, \dots, S}$ $S$ discrete probability measures (corresponding to $S$ point clouds), the barycenter of $(\mu_s)$ with non-negative weights $(\lambda_s)_{s=1, \dots, S}$ is the solution of the following minimization problem:
\begin{equation}
\min_{\mu} \sum_{s=1}^S \lambda_s \Wass_2^2(\cdot, \mu_s).    
\end{equation} 
This  problem does not have an explicit solution, and an optimization algorithm must be run  every time the weights are changed. Using transport maps from a reference measure $\rho$, it is natural to consider instead 
\begin{equation}
\mu = \left(\sum_{s=1}^S \lambda_s T_{\mu_s}\right)_{\#} \rho    
\end{equation} 
as the barycenter of the $(\mu_s)$, and one can indeed check that the measure $\mu$ defined
by this formula minimizes
$\sum_{s} \lambda_s W^2_{2,\rho}(\cdot,\mu_s).$  We illustrate this idea with
the computation of barycenters of $4$ point clouds in Figure \ref{fig:barycenter}. Again,
in practice operations are performed on the vectorized Monge maps $\bm{T_\mu}$.

These barycenters are in general not equal to their Wasserstein counterparts but they seem to retain  the geometric information contained in the point clouds. This idea can be used to extend  unsupervised learning algorithms such as $k$-Means to family of point clouds. As a toy example, we perform a clustering on the images of the MNIST dataset \cite{lecun-mnisthandwrittendigit-2010}. We convert the $60,000$ images of the training set into point clouds of $\X = [0,1]^2$ using a simple thresholding on the pixels intensity and we compute for each point cloud its Monge map embedding. We then perform a clustering with the $k$-means++ algorithm \cite{Arthur:2007:KAC:1283383.1283494} on the vectorized Monge maps, looking for $k = 20$ clusters. Figure \ref{fig:kmeans} shows the push-forwards of the $20$ centroids in $L^2(\rho,\Rsp^d)$. 
\begin{figure}
    \centering
    \includegraphics[width=.8\textwidth]{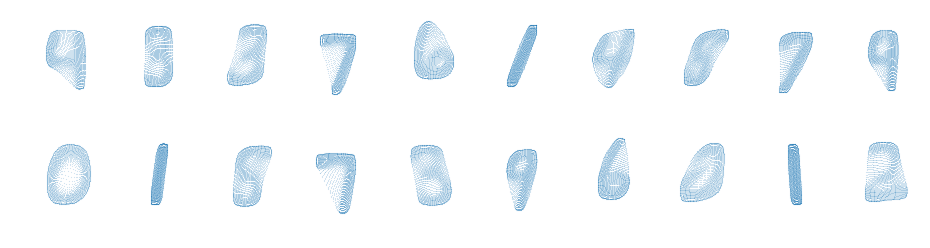}  
    \caption{Push-forwards of the $20$ centroids after clustering of the Monge map embeddings of the MNIST training set.}
    \label{fig:kmeans}
\end{figure}

\section{CONCLUSION}
We have shown that measures can readily be embedded explicitly in a Hilbert space by their optimal transport map between an arbitrary reference measure and themselves.
These embeddings are shown to be injective and bi-Hölder continuous w.r.t the Wasserstein  distance. They enable the definition of distances between measures and the use of generic machine learning algorithms in a computationally tractable framework.

Future work will focus on the extension of the stability theorem to more general sources and costs, to the improvement of the Hölder exponent and to statistical properties of transport plans, including concentration bounds and sample complexity of the distance they define.





\subsection*{Acknowledgement} 
The first author warmly thank Clément Cancès for pointing Lemma~3.7 in
\cite{eymard2000finite} and Robert Berman for discussions related to
the topic of this article, and acknowledges the support of the Agence national de la
recherche through the project MAGA (ANR-16-CE40-0014).

\bibliographystyle{plain}
\bibliography{ref}

\appendix
\section{Proof}

\subsection*{Proof of Corollary~\ref{coro:Berman}}

We first state a simple lemma that links the uniform norm of Lipschitz function to its $L^2(\rho)$ norm:
\begin{lemma}  
\label{lemma:lip-l2}
If $f$ is $L$-Lipschitz on a convex bounded domain $\X$, then
\begin{equation} 
\|f\|_{\infty} \leq C \|f\|_{\L^2(\X)}^{\frac{2}{d+2}}
\end{equation}
for some $C$ depending on $L$, $d$ and $\X$ only.
\end{lemma}


\begin{proof}[Proof of Corollary~\ref{coro:Berman}]
Theorem \ref{th:Berman-original} implies
\begin{equation}
     \| \nabla \psi_\mu - \nabla \psi_\nu \|^2_{\L^2(\Y)} \leq C \Big(\int_{\Y} (\psi_\nu - \psi_\mu)\dd(\mu - \nu)\Big)^{\frac{1}{2^{d-1}}}, 
\end{equation}
and as in Theorem \ref{th:hold-cont-reg}, the quantity in the parenthesis can be upper bounded by $2 M_\X \Wass_1(\mu, \nu)$. Adding a constant to $\psi_\mu$ if necessary, we can assume
that $\int_{\Y} \psi_\mu(y)\dd y = \int_{\Y} \psi_\nu(y)\dd y$.
The Poincaré-Wirtinger inequality on $\Y$  then implies 
\begin{equation}
    \|\psi_\mu - \psi_\nu \|^2_{\L^2(\Y)} \leq C' \Wass_1(\mu, \nu)^{\frac{1}{2^{d-1}}} 
\end{equation}
for some $C'$ depending only on $\rho$, $\X$ and $\Y$.

We reuse the fact that $\psi_\mu - \psi_\nu$ is Lipschitz with constant  $\leq2 M_\X$ to use Lemma \ref{lemma:lip-l2}:
\begin{equation}
    \|\psi_\mu - \psi_\nu \|_{\infty} \leq C \Wass_1(\mu, \nu)^{\frac{2}{2^{d-1}(d+2)}}.
\end{equation} 
Since $\phi_\mu = \psi_\mu^*$ and $\phi_\nu = \psi_\nu^*$, the definition of the Legendre transform \eqref{eq:lega} yields
\begin{equation} 
\|\phi_\mu - \phi_\nu \|_{\infty} \leq \|\psi_\mu - \psi_\nu \|_{\infty}\leq C \Wass_1(\mu, \nu)^{\frac{2}{2^{d-1}(d+2)}}. 
\end{equation}
We conclude using Proposition \ref{prop:chazal} and the fact that $\phi_\mu$ is $\diam(\Y)$-Lipschitz:  there exists a constant $C$ depending only on $\rho$, $\X$ and $\Y$ such that 
\begin{equation*}
    \|T_\mu - T_\nu\|_{\L^2(\rho)} \leq C \Wass_1(\mu, \nu)^{\frac{1}{2^{(d-1)}(d+2)}}
    \qedhere
\end{equation*} 
\end{proof}


\subsection*{Proof of Proposition \ref{prop:poincare}}
This proof is a straightforward adaptation of a stability result for
finite volume discretization of elliptic PDEs, see Lemma~3.7 in
\cite{eymard2000finite}.  We consider the function $u$ on
$\X$ defined a.e. by $\restr{u}_{V_i(\bm{\psi})} = v_i$. Then,
\begin{equation}
\sca{v^2 - \sca{v}{G(\bm{\psi})}^2}{G(\bm{\psi})} = \int_\X  u^2 - \left(\int_\X u\right)^2 = \frac{1}{2} \int_{\X\times \X} (u(x) - u(y))^2 \dd y \dd x 
\end{equation}
so it suffices to control the right hand side of this equality. Given $(i,j) \in \{1,\hdots,N\}$ and $(x,y) \in X$, we denote 
\begin{equation}
    \chi_{ij}(x,y) = \begin{cases} 1&\hbox{ if } V_i(\bm{\psi})\cap V_j(\bm{\psi})\cap [x,y] \neq \emptyset \hbox{ and } \sca{y_j - y_i}{y-x}\geq 0 \\
0 &\hbox{ if not.} \end{cases}
\end{equation}
Then, $ u(y) - u(x) = \sum_{i\neq j} (v(y_j) - v(y_i)) \chi_{ij}(x,y)$. We introduce
\begin{equation}
d_{ij} = \nr{y_j - y_i}, \qquad c_{ij,z} = \abs{\sca{\frac{z}{\nr{z}}}{\frac{y_j-y_i}{\nr{y_j-y_i}}}}, 
\end{equation}
and we  apply Cauchy-Schwarz's inequality to get
\begin{align}
    (u(y) - u(x))^2 &= \left(\sum_{i\neq j} (v(y_j) - v(y_i)) \chi_{ij}(x,y)\right)^2 \notag\\
    &\leq 
    \sum_{i\neq j} \frac{(v(y_j) - v(y_i))^2}{d_{ij} c_{ij,y-x}} \chi_{ij}(x,y)\sum_{i\neq j} d_{ij} c_{ij,y-x} \chi_{ij}(x,y) 
\end{align}
In addition, when $\chi_{ij}(x,y) = 1$,  we have $\sca{y-x}{y_j - y_i} \geq 0$ so that
\begin{equation} d_{ij} c_{ij,y-x} = \nr{y_j - y_i} \sca{\frac{y-x}{\nr{y-x}}}{\frac{y_j-y_i}{\nr{y_j-y_i}}} \geq 0
\end{equation}
and 
\begin{equation}\sum_{i\neq j} d_{ij} c_{ij,y-x} \chi_{ij}(x,y) = \sum_{i\neq j} \sca{\frac{y-x}{\nr{y-x}}}{y_j - y_i}\chi_{ij}(x,y) \leq \diam(\Y). 
\end{equation}
Therefore, 
\begin{align}
    &\int_{\X \times \X} (u(y) - u(x))^2\dd x \dd y  \notag\\
    &\leq \diam(\Y)\int_{\X \times \X} \sum_{i\neq j} \frac{(v(y_j) - v(y_i))^2}{d_{ij} c_{ij,y-z}} \chi_{ij}(x,y) \dd x \dd y \notag\\
    &= \diam(\Y)\int_{B(0, \diam(\X))} \sum_{i\neq j} \frac{(v(y_j) - v(y_i))^2}{d_{ij} c_{ij,z}} \left(\int_{\X} \chi_{ij}(x,x+z) \dd x\right) \dd z
\end{align}
Moreover, denoting $m_{ij} = \vol^{d-1}(V_i(\bm{\psi}) \cap V_j(\bm{\psi}))$ we get
\begin{equation} 
\int_{\X} \chi_{ij}(x,x+z) \dd x \leq m_{ij} \nr{z} c_{ij,z}  
\end{equation}
thus giving 
\begin{equation} 
\int_{\X \times \X} (u(y) - u(x))^2\dd x \dd y \leq C(d) \diam(\Y)\diam(\X)^{d+1} \sum_{i\neq j} \frac{m_{ij}}{d_{ij}}(v(y_j) - v(y_i))^2 
\end{equation}
Define $H_{ij} = \frac{m_{ij}}{d_{ij}}$, $H_{ii} = -\sum_{j\neq i} H_{ij}$. Then, $\DD G(\bm{\psi}) = H$, and
\begin{align}
    \sca{\DD G(\bm{\psi})v}{v} &= \sum_{i,j} H_{ij}v_i v_j  \notag\\
    &= \sum_{i} \left(H_{ii} v_i v_i + \sum_{j\neq i} H_{ij} v_i v_j\right) \notag\\
    &= \sum_{i} \sum_{j\neq i} H_{ij} v_i (v_j - v_i) \notag\\
    &= \sum_{j\neq i} H_{ij} v_i (v_j - v_i) := A. 
\end{align} 
And 
\begin{equation} \sum_{i\neq j} H_{ij}(v(y_j) - v(y_i))^2 = \sum_{i\neq j} H_{ij} v_j (v_j - v_i) - \sum_{i\neq j} H_{ij} v_i (v_j - v_i) = -2A.
\end{equation}
We finally obtain 
\begin{equation}
\iint (u(y) - u(x))^2\dd x \dd y \leq - C_{d, \X, \Y} \sca{\DD G(\bm{\psi}) v}{v}.
\qedhere\end{equation}

\end{document}